\documentclass[runningheads,a4paper]{llncs}
\pdfoutput=1

\usepackage{amssymb}
\usepackage{amsmath}
\usepackage{graphicx}
\usepackage[caption=false]{subfig} 
\usepackage{url}
\usepackage{booktabs}

\newcommand{\keywords}[1]{\par\addvspace\baselineskip\noindent\keywordname\enspace\ignorespaces#1}

\title{A Framework for Constrained and Adaptive Behavior-Based Agents}
\author{Renato de Pontes Pereira, Paulo Martins Engel}
\institute{Informatics Institute \\
Federal University of Rio Grande do Sul, UFRGS\\
Porto Alegre, RS, Brazil, CEP 91501-970\\
\path|{rppereira, engel}@inf.ufrgs.br|}

\begin{document}

\mainmatter  
\toctitle{Lecture Notes in Computer Science}
\tocauthor{Authors' Instructions}
\maketitle

\begin{abstract}

Behavior Trees are commonly used to model agents for robotics and games, where constrained behaviors must be designed by human experts in order to guarantee that these agents will execute a specific chain of actions given a specific set of perceptions. In such application areas, learning is a desirable feature to provide agents with the ability to adapt and improve interactions with humans and environment, but often discarded due to its unreliability. In this paper, we propose a framework that uses Reinforcement Learning nodes as part of Behavior Trees to address the problem of adding learning capabilities in constrained agents. We show how this framework relates to Options in Hierarchical Reinforcement Learning, ensuring convergence of nested learning nodes, and we empirically show that the learning nodes do not affect the execution of other nodes in the tree.

\keywords{Behavior Trees, Reinforcement Learning, Hierarchical Reinforcement Learning, Agent Modeling, Robotics, Games}
\end{abstract}

\section{Introduction}
\label{sec:introduction}

Some applications require a strict control of autonomous agents, i.e., the agent must be reliable to perform the right action in the right moment to achieve its goals, with minimum (ideally zero) chance to fail. 
For example, consider a large robot that interacts with humans in an industry. This robot must operate very carefully in order not to damage other equipments neither harm humans. 

A Behavior Tree (BT) is a plan representation and decision making tool for modeling autonomous agents mainly used in the video games industry, but gaining attention in robotics, control theory and general agent modeling. BTs are appealing because they allow human experts to design \emph{constrained behaviors}. We define a constrained behavior as a control block that guarantees the execution of a specific chain of actions given a specific set of perceptions. 

Behavior Trees provide several powerful features to model constrained agents, such as: reliability that the agent will perform exactly the behaviors designed by the expert; flexibility to reuse behaviors, and build and maintain large models without losing readability; and an easiness to debug and identify possible errors. 
However, Behavior Trees also have limitations, for example: 1) in general, a good agent must have a large set of behaviors in order to respond differently to the events in the environment, but a BT model depends entirely on the expert manual work, therefore, creating good agents demands much time and effort from the expert to design and test this large set of behaviors; 2) the lack of variation in behaviors (due to the cost in terms of time or effort to make more behaviors) makes the agent predictive and repetitive, which can be a negative factor in applications where the agent interacts with humans; 3) the Behavior Tree does not have any adaptation mechanism, therefore, agents cannot adapt to changes in the environment neither improve the initial expert design, this can also be an issue when agents interact with humans. 

One way to bring adaptiveness and avoid repetitiveness is using learning algorithms. Reinforcement Learning (RL), in particular, can be used to model learning agents that can adapt to unknown environments and optimize their performance in real time and in an online way. However, learning is often discarded in applications that require constrained agents because it can bring harmful problems to the agent and to the task, some problems include: it may not guarantee convergence or stability, making the agent unreliable; it may require too much data or time to be trained; large robots, for example, may bring danger to equipments or living beings; the robot itself can be damaged; characters controlled by computer in competitive games, for example, may learn bad actions from humans; it may not generalize, performing poorly in other environments than those presented on training; may converge to sub-optimal behaviors and never change again.

In this paper we address the problem of adding learning capabilities to constrained agents by using Reinforcement Learning together with Behavior Trees, keeping the advantages and features of BTs and RL while minimizing the risks of learning. To achieve this, we define a new node in the Behavior Tree, the \emph{learning node}, which embeds a local Reinforcement Learning model without changing the overall tree structure. Based on this node, we propose a general framework for modeling constrained yet adaptive agents that is related to the Options framework in Hierarchical Reinforcement Learning. 

The remainder of this paper is structured as follows. Section \ref{sec:bt} presents an overview of Behavior Trees, together with the formal definition of this tool. Section \ref{sec:rl} describes the basics of the Reinforcement Learning through Markov Decision Processes and expands it to Hierarchical Reinforcement Learning and Semi-Markov Decision Processes. Section \ref{sec:learningnodes} presents the basis of the proposed framework and a proof that it is related to the Options framework of Hierarchical Reinforcement Learning. The proposed framework is also validated empirically in Section \ref{sec:experiments} with two agent simulation experiments using a fire control scenario. Section \ref{sec:related} discusses the relation of this framework with other related works on Behavior Trees and Reinforcement Learning. Finally, Section \ref{sec:conclusion} presents the final notes and discusses the future works.

\section{Behavior Trees}
\label{sec:bt}

A Behavior Tree (BT) is a plan representation and decision making tool used to model and control autonomous agents. This tool was created in the game industry \cite{isla2005} with fast adoption by the game development community. BTs are commonly used to model NPCs (Non-Playable Characters or Characters Controlled by the Computer) and they are viewed as an alternative to FSMs (Finite State Machines), HFSMs (Hierarchical Finite State Machines) and hand-coded rules via scripting. Since 2012, there are some efforts to apply BTs on robotics, including modeling and controlling of UAVs \cite{ogren2012}, fault tolerance in hybrid and multi-robot systems \cite{colledanchise2014b}\cite{colledanchise2015} and an attempt to unify notation and formalize Behavior Trees as controller, proving the relation to CHDSs (Controlled Hybrid Dynamical Systems) \cite{marzinotto2014}. We follow the description proposed by \nocite{marzinotto2014}(Marzinotto et al., 2014), which is the most formal and coherent description of Behavior Trees in the current literature, but with a few minor changes (please refer to Section \ref{sec:related} for a comparison).

Behavior Trees provide several features that are important for real-time, constrained and complex application:

\begin{itemize}
  \item \textbf{readability}: all information, transitions, connections, relations and conditions modeled by BTs are explicit and compact. This is especially useful for maintenance and collaborative work;

  \item \textbf{maintainability}: because transitions in BT are defined by the structure in an explicit way (the model is not black box), nodes can be designed independently from one another, thus, when adding, modifying or removing nodes (or subtrees) it is not necessary to change other pieces of the model.

  \item \textbf{scalability}: when a BT has many nodes, it can be decomposed into small subtrees saving the readability of the graphical model.

  \item \textbf{reusability}: due to the independence provided by BT, the nodes and subtrees can be reused in other models and projects.

  \item \textbf{goal-oriented}: although the nodes are independent from one another, they still are related by the structure of the model. This allows designers to build specific sub-trees for a given goal without losing flexibility of the model.

  \item \textbf{parallelization}: parallelization is possible (and easy to do) because all worker processes are locally contained to the parallel node.
\end{itemize}

We define a Behavior Tree as a Directed Rooted Tree $\mathcal{G}(\mathcal{V}, \mathcal{E})$ with $|\mathcal{V}|$ nodes and $|\mathcal{E}|$ directed edges. For a pair of connected nodes, the outgoing node is called \emph{parent} and the incoming node is called \emph{child}. The child-less nodes are called \emph{leaves} while the parent-less node is called \emph{root}, which must have only a single child. The nodes standing between the root and the leaves are called \emph{internal}. Each subtree in the model defines a different \emph{behavior}. A single leaf node is called a \emph{primitive behavior} (also called \emph{trivial} or \emph{atomic} behavior). \emph{Composed behaviors} use combinations of \emph{primitive behaviors} and other composed ones, defining a \emph{behavior hierarchy}. 

Periodically, the \emph{root} generates a \emph{tick} signal and propagate it through the tree branches according to the algorithm defined by each node type. When the \emph{tick} reaches a leaf node, a computation is made and the node returns a \emph{state} value: \emph{SUCCESS}, \emph{FAILURE} or \emph{RUNNING}\footnote{We also use a special \emph{state} called \emph{ERROR}, which is always propagated immediately back to the root (similar to \emph{RUNNING}). This \emph{state} is not always described as part of BTs core, but sometimes is useful for debugging.}. Then the returned \emph{state} value is propagated back and forth through the tree according to the algorithm defined by each node type. The \emph{tick} signal stops when it reaches the \emph{root}.

Each node in the tree, except the \emph{root}, belongs to one of the following categories: \emph{Composite}, \emph{Decorator}, \emph{Action} or \emph{Condition}. Figure \ref{fig:node_categories} presents the visual example of them. Nodes of each one of these categories have a specific responsibility and constraints in the tree:

\begin{itemize}
  \item \textbf{Composite}: nodes of this category are commonly referred to as control-flow nodes, because their role is to propagate the tick signal to its children, respecting some defined order. Composite nodes must decide which child will be ticked and which \emph{state} value will be returned. A composite node can have one or more children and preferably does not perform any computation more than the necessary to choose the children to tick. Composites are represented graphically by squares with a symbol or rectangles with text.

  \item \textbf{Decorator}: the goal of decorator is to change the behavior of its child by manipulating the returned \emph{state} value or changing the child ticking frequency. For example, the decorator may invert the result state of its child, similarly to the NOT operator, or it can repeat the execution of the child for a given number of times. Decorators have only a single child and are represented graphically by rhombuses.

  \item \textbf{Action}: action nodes are obligatorily leaves, they do not propagate the \emph{tick} signal. Instead, they perform some computation to change the environment or the internal state of the system, and return a state value. Actions of a robot may involve sending signals to the engines, playing sound through speakers or turning on lights, while the actions of a NPC may involve executing animations, performing spatial transformations, playing sounds, etc. Actions are represented graphically by rectangles.

  \item \textbf{Condition}: like action nodes, conditions are obligatorily leaves and perform some computation instead of propagating the \emph{tick} signal. The difference is that, a condition node does not change the environment or any internal variable in the system, it only checks whether a certain condition has been met or not. To accomplish that, the node commonly has a target variable and a criterion to base the decision. Examples of condition nodes: ``is obstacle close?'', ``is other agent visible?'', ``is battery low?'' or ``am I hungry?''. Conditions are represented graphically by ellipses.
\end{itemize}

\begin{figure}
  \centering
  \includegraphics[width=\textwidth]{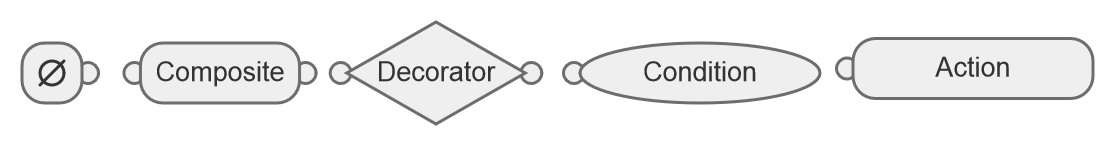}
  \caption{The visual representation of each node category (and the Root). From left to right: Root, Composite, Decorator, Condition and Action}
  \label{fig:node_categories}
\end{figure}

An action returns \emph{SUCCESS} if it could be completed; it returns \emph{FAILURE} if, for any reason, it could not be finished; or returns \emph{RUNNING} if the action is still executing. A condition node returns \emph{SUCCESS} if the condition has been met, otherwise it returns \emph{FAILURE}. It never return \emph{RUNNING}. Composite and Decorators handle \emph{SUCCESS} and \emph{FAILURE} differently, but, in general, they return \emph{RUNNING} as soon as they receive it.

There are no common core decorator, action or condition node type, they depend on the task. However, there are five core composite node types: \emph{Sequence}, \emph{Priority}, \emph{MemSequence}, \emph{MemPriority} and \emph{Parallel}. Figure \ref{fig:composite_types} shows the visual of these nodes. There can be other types depending on the application of the model, but these are generic and necessary nodes to build a complete controller \cite{marzinotto2014}. Each node type works in a specific way:

\begin{itemize}
  \item \textbf{Sequence}: when a Sequence node is \emph{ticked}, it propagates the \emph{tick} signal to its children sequentially. If any child returns \emph{FAILURE} or \emph{RUNNING}, the Sequence stops the propagation and returns the received \emph{state}. However, if all Sequence children return \emph{SUCCESS}, the Sequence also returns \emph{SUCCESS}. This node is represented graphically by an arrow $\rightarrow$.

  \item \textbf{Priority}: a Priority node (sometimes called Selector) also propagates the \emph{tick} signal to its children sequentially when the node is \emph{ticked}. If any child returns \emph{SUCCESS} or \emph{RUNNING}, the Priority stops the propagation and returns the received \emph{state}. If all children return \emph{FAILURE}, the Priority also returns \emph{FAILURE}. This node is represented graphically by a question marker $?$.

  \item \textbf{MemSequence}: the MemSequence is a version of Sequence with memory. It works in a similar way to Sequence, but when a child $i$ returns \emph{RUNNING}, the MemSequence node will not propagate the next \emph{tick} to nodes preceding $i$ (i.e., it will only \emph{tick} $i$, $i+1$, $i+M$). This node is represented graphically by an arrow $\rightarrow$ together with an asterisk $*$.

  \item \textbf{MemPriority}: MemPriority is the memory version of Priority. It does the same as MemSequence does for Sequence. This node is represented graphically by a question marker $?$ together with an asterisk $*$.

  \item \textbf{Parallel}: when a Parallel node is \emph{ticked}, it propagates the \emph{tick} to all its children at the same time. Then it returns \emph{SUCCESS} if $S$ children or more return \emph{SUCCESS}; it returns \emph{FAILURE} if $F$ or more return \emph{FAILURE}. Otherwise, it returns \emph{RUNNING}. This node is represented graphically by a double arrow $\rightrightarrows$.
\end{itemize}

\begin{figure}
  \centering
  \includegraphics[width=0.65\textwidth]{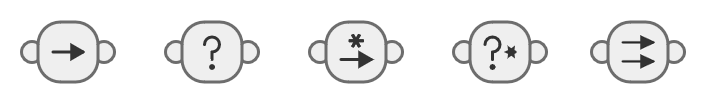}
  \caption{The visual representation of the core node types. From left to right: Sequence, Priority, MemSequence, MemPriority, Parallel}
  \label{fig:composite_types}
\end{figure}


Notice that, with the exception of Parallel, all composite nodes propagate the \emph{tick} sequentially, creating the notion of priority among behaviors in the tree. There are two common ways to represent a Behavior Tree in a 2D layout: horizontally and vertically. A composite node \emph{ticks} its children from left to right in a vertical view, or from top to bottom in a horizontal view. Expanding this idea, we can say that the left/top branch of the tree (starting from the root) is the behavior with the highest priority, while the right/bottom branch of the tree has the lowest priority. This characteristic is essential for a controller, e.g., avoiding collisions or falls is more important than exploring the environment. Please refer to the figures at Section \ref{sec:experiments} to real examples of a Behavior Tree model.

\section{Reinforcement Learning}
\label{sec:rl}

We address the problem of adding capacity of learning to an agent that acts without supervision from experts (i.e., input-output pairs are never presented) in real-time applications (i.e., it must learn continuously) and preferably in an online fashion (i.e., update the model upon the arrival of a new piece of information). This problem can be formulated as a Markov Decision Process (MDP) problem, which is the common basis for the Reinforcement Learning task.

In an MDP problem, a learning agent interacts with an environment with discrete time steps $t = 1, 2, 3, ...$. On each time step, the agent observes the system state $s_t \in \mathcal{S}$ and performs and action $a_t \in \mathcal{A}_{s_t}$, where $\mathcal{A}_s$ is a finite and non-empty set of admissible actions for a state $s$. As consequence of the agent action, the system generates a reward $r_{t+1}$ according to an expected immediate reward function $R(s, a) = \mathbb{E}\{r_{t+1}|s_t=s, a_t=a\}$ and a next state $s_{t+1}$ with a transition probability of $P(s'|s, a) = P(s_{t+1}=s'|s_t=s, a_t=a)$.

The goal of the agent is to learn a policy $\pi : \mathcal{S} \times \mathcal{A} \rightarrow [0, 1]$, where $A = \bigcup_{s \in \mathcal{S}}\mathcal{A}_s$ (the set of all actions), that maximizes the expected discounted future reward from each state $s$ \cite{sutton1999}:

\begin{align}
  V^\pi(s) &= \mathbb{E}\{r_{t+1} + \gamma r_{t+2} + \gamma^2 r_{t+3} + \ldots|\pi, s_t=s \} \label{eq:valuefunction} \\
           &= \mathbb{E}\{r_{t+1} + \gamma V^\pi(s_{t+1})|\pi, s_t=s\} \notag \\
           &= \sum_{a \in \mathcal{A}_s} \pi(s, a) \left[ R(s, a) + \gamma\sum_{s'} P(s'|s, a)V^\pi(s') \right], \label{eq:valuebellman}
\end{align}

where $\pi(s, a)$ is the probability of a policy $\pi$ choosing an action $a$ in the state $s$, and $\gamma \in [0, 1]$ is a discount-rate parameter. $V^\pi(s)$ is called the value function for $\pi$, and denotes the value of the state $s$ when following a policy $\pi$. An optimal policy $\pi^*$ is any policy that corresponds to the unique optimal value function $V^*(s)$, defined as:

\begin{align}
  V^*(s) &= \max_\pi V^\pi(s) \label{eq:optimalvaluefunction} \\
         &= \max_{a \in \mathcal{A}_s} \mathbb{E}\{r_{t+1} + \gamma V^\pi(s_{t+1})|s_t=s, a_t=a\} \notag \\
         &= \max_{a \in \mathcal{A}_s} \left[ R(s, a) + \gamma \sum_{s'} P(s'|s, a)V^*(s') \right]. \label{eq:optimalvaluebellman}
\end{align}

Unfortunately, the only solutions to Equations \ref{eq:valuebellman} and \ref{eq:optimalvaluebellman} are the Equations \ref{eq:valuefunction} and \ref{eq:optimalvaluefunction} \cite{sutton1998}, respectively, which cannot be computed without knowing $R(s, a)$ and $P(s'|s, a)$. Alternatively, we can define a state value function $Q^\pi(s, a)$ that works upon pairs of states and actions, rather than just states. $Q^\pi(s, a)$ denotes the value of taking an action $a$ while in a state $s$ under a policy $\pi$, therefore:

\begin{align}
  Q^\pi(s, a) &= \mathbb{E}\{r_{t+1} + \gamma r_{t+2} + \gamma^2 r_{t+3} + \ldots|\pi, s_t=s, a_t=a \} \notag \\ 
              &= R(s, a) + \gamma \sum_{s'} P(s'|s, a)V^\pi(s') \notag \\
              &= R(s, a) + \gamma \sum_{s'} P(s'|s, a) \sum_{a'}\pi(s', a')Q^\pi(s', a'), \notag 
\end{align}

and its respectively optimal value $Q^*(s, a)$:

\begin{align}
  Q^*(s, a) &= \max_\pi Q^\pi(s, a) \notag \\
            &= R(s, a) + \gamma \sum_{s'} P(s'|s, a) \max_{a'} Q^*(s', a'). \notag
\end{align}

Supposing that at each state $s$, the agent performs an action $a$ and receives a reward $r$, and then observes the new state $s'$. We can approximate $Q^*(s, a)$ by updating an estimate $Q_k(s, a)$:

\begin{equation}
  Q_{k+1}(s, a) = (1 - \alpha_k) Q_k(s, a) + \alpha_k \left[ r + \gamma \max_{a' \in \mathcal{A}_{s'}} Q_k(s', a') \right],
\end{equation}

where $\alpha_k$ is a time-varying learning-rate parameter. This process is called Q-learning, which is a classical and widely used Reinforcement Learning algorithm. With this update process, ${Q_k}$ converges to $Q^*$ with probability 1 if at the limit, all admissible state-action pairs are updated infinitely often, and $\alpha_k$ decays with time, regardless of the policy being followed \cite{sutton1998}\cite{barto2003}.

\subsection{Hierarchical Reinforcement Learning}

While most classical reinforcement learning relies on the configuration described by MDPs, Hierarchical Reinforcement Learning (HRL) models have their basis on Semi Markov Decision Processes (SMDPs) \cite{barto2003}. Due to the generalization of SMDPs, these hierarchical models are able to explore the temporal aspects of the tasks and reduce the impact of the curse of dimensionality by splitting the task space into several subproblems.

In a SMDP, the learning agent considers the time during transition between states. This means that, when observing a state $s$, an agent performs an action $a$ that takes $\tau$ time steps\footnote{Here, we assume $\tau$ to be discrete, but it can be extended to continuous without much impact \cite{barto2003}} to  move to a new state $s'$. Now, the joint probability is rewritten to $P(s', \tau|s, a)$ and the expected immediate reward function $R(s, a)$ now gives the amount of discounted reward expected to accumulate over the waiting time in $s$ given $a$ \cite{barto2003}. Equations for the optimal value function $V^*$ and the optimal state value function $Q^*$ are also rewritten to:

\begin{equation}
  V^*(s) = \max_{a \in \mathcal{A}_s} \left[ R(s, a) + \sum_{s', \tau} \gamma^\tau P(s', \tau|s, a)V^*(s') \right],
\end{equation}

and

\begin{equation}
  Q^*(s, a) = R(s, a) + \sum_{s', \tau} \gamma^\tau P(s', \tau|s, a) \max_{a'} Q^*(s', a'), 
\end{equation}

respectively.

One successful approach based on SMDPs is the \emph{Option} framework proposed in \nocite{sutton1999} (Sutton et al., 1999). An option is a generalization of an action in a way that it can call other options upon the execution, creating the idea of a hierarchy. When an option is initiated, it can call another option, then this new option can call another one, and so on until it finds a \emph{primitive option} (the actions of the MDP framework). An option can be defined as 3-tuple $\langle \mathcal{I}, \mu, \beta \rangle$, consisting of an input set $\mathcal{I} \subseteq \mathcal{S}$, a semi-Markov policy $\mu : \mathcal{S} \times \mathcal{O} \rightarrow [0, 1]$ (where $\mathcal{O} = \bigcup_{s \in \mathcal{S}} \mathcal{O}_s$), and a termination condition $\beta : \mathcal{S} \rightarrow [0, 1]$. A given option can only be initiated if, and only if, the current state $s$ is an element of $\mathcal{I}$. While executing, it chooses the next option $o$ with probability $\mu(s, o)$, the environment changes to state $s'$, where the option terminates with probability $\beta(s')$. Notice that $\mu$ is a semi-Markov policy over policies, i.e., it can choose the next option based on the entire history $h$ of states, actions, and rewards since the option was initiated \cite{barto2003}.

With this approach, the option-value function for $\mu$ is:

\begin{equation}
  Q^\mu(s, o) = \mathbb{E}\{r_{t+1} + \gamma r_{t+2} + \ldots + \gamma^{\tau-1} r_{t+\tau} + \ldots| \mathcal{E}(o\mu, s, t) \},
\end{equation}

where $\mathcal{E}(o\mu, s, o)$ is the event of $\mu$ being initiated at time $t$ in $s$, and $o\mu$ is the semi-Markov policy that follows $o$ until it terminates after $\tau$ time steps and then continues according to $\mu$. Respectively, the optimal function are rewritten to:

\begin{equation}
  Q^*_\mathcal{O}(s, o) = R(s, o) + \sum_{s'} P(s'|s, o) \max_{o' \in \mathcal{O}_{s'}} Q^*_\mathcal{O}(s', o'), 
\end{equation}

where:

\begin{equation}
  P(s'|s, o) = \sum^{\infty}_{\tau=1} p(s', \tau)\gamma^\tau
\end{equation}

for all $s \in \mathcal{S}$, where $p(s', \tau)$ is the probability that $o$ terminates in $s'$ after $\tau$ steps when initiated in state $s$. The corresponding Q-learning update is:

\begin{equation}
  Q_{k+1}(s, o) = (1 - \alpha_k) Q_k(s, o) + \alpha_k \left[ r + \gamma^\tau \max_{o' \in \mathcal{O}_{s'}} Q_k(s', o') \right],
\end{equation}

\section{Learning Framework}
\label{sec:learningnodes}

We propose a framework for modeling agents that follow behaviors strictly as modeled by a human expert, and still are able to learn from experience. This framework allows an expert to manually define which behaviors an agent will have, to specify when and in which order the behaviors will be executed, and to specify where learning can be applied. Behavior Trees are used as the base modeling tool for our framework, due to their advantages such as described in Section \ref{sec:bt}. In summary, BTs are compact, easy to understand, maintain and reuse, they scale well, have expressive power and can be easily parallelized. We also adopt Reinforcement Learning in order to support our modeling tool, providing the capacity to learn in real time with the agent experiences. With RL, agents can optimize actions in specific situations and be able to adapt to changes on the environment configuration and dynamics.

In order to make Reinforcement Learning work together with Behavior Trees we propose the use of \emph{Learning Nodes}. Within this approach, Reinforcement Learning can be embedded into Behavior Trees in a modular and reusable way. We propose two uses of these custom nodes: as action and as composite.

In the Learning Action Node, the expert must choose how to represent the state $s$, the actions $A$, and the reward $r$, according to the task. Suppose, for example, a robot with an action node to ``grab an object''. This node can use RL in order to learn how to grab different kinds of objects or improve how to grab a known object in different positions. In this case, the state could be the position of the object relative to the robot's hand, the actions could be the different joint configurations in the robot's arm, and the reward function would return a positive value if the robot could grab the object, otherwise it would return a negative value.

In the Learning Composite Node, the expert must also choose how to represent the state $s$, but the actions $A$ are the children of the composite node. So, given a state $s$, the Learning Composite Node selects among the $N$ children $c_1, c_2, \ldots, c_N$. The reward function could be provided according to the task or it could use the state value returned from children ticks (i.e., positive reward for \emph{SUCCESS} and negative reward for \emph{FAILURE}). 

As an application example of a Learning Composite Node, consider an agent in a life simulation. The agent has 5 major behaviors (those which are connected directly to the root's unique child): ``find food'', ``eat'', ``rest'', ``hide''and ``run from predators''. These behaviors can be simple action nodes or  complex subtrees with several other nodes. For this example, the goal of this agent is to learn when to use these behaviors in order to maximize the chance to survive. In this case, these behaviors would be children of a Learning Composite Node, thus being the actions $A$. The state $s$ could be a series of variables, such as the presence of predator, distance from food, hungry level, health level, etc. The reward function of this example could the a combination of one or more variables, such as the hungry level, fatigue, health level, etc.

To make a formal definition of these two nodes, we exploit the Options approach for Hierarchical Reinforcement Learning. Reminding that an option $\langle \mathcal{I}, \mu, \beta \rangle$ is only initiated if the state $s \in \mathcal{I}$, and while executing the option uses the policy $\mu$ to decide which option $o$ (action) will be executed, based on the history $h$ of past states, actions and rewards since the beginning of execution of that option. After that, the environment generates a new state $s'$ with probability $P(s'|s, o)$. We argue that, the Behavior Tree defined here can be modeled as a specialization of an Option-based Hierarchical Reinforcement Learning.

\begin{theorem}
\label{theorem:options}
A Behavior Tree with core nodes is a specialization of Options in Hierarchical Reinforcement Learning.
\end{theorem}
 
\begin{proof}
Following the definition presented in Section \ref{sec:rl}, an option has the following characteristics: 1) it is a hierarchical combination of other options or it is a primitive option (action of MDP); 2) it has an input set $\mathcal{I}$, a termination condition $\beta$ and a policy $\mu$; 3) the policy $\mu$ can choose other options based on the history $h$ of past states, actions and rewards since the beginning of execution of that option. In Section \ref{sec:bt}, we defined a hierarchical behavior exactly as an option. All nodes in a BT have an input set $\mathcal{I} = \mathcal{S}$ because they are not evaluated before the execution, thus the input set is equal to the whole state set. All nodes also have a termination condition $\beta$ defined by each node type (e.g., the termination condition of a Priority node is: one child returning \emph{SUCCESS} or all children returning \emph{FAILURE}). Composite and decorator nodes have a fixed policy $\mu$ that always calls children sequentially, while actions and conditions are related to primitive options (does not have a policy, instead they perform some interaction with the environment). The core composite nodes depend on the temporal aspect of SMDPs, because their selection of a child $c_i$ depends the execution of all children $c_0, \ldots, c_{i-2}, c_{i-1}$.
\end{proof}

Given this theorem, we can define:

\begin{definition}
A Learning Composite Node can be seen as an option with $\mathcal{I} = \mathcal{S}$, with children $c_1, c_2, \ldots, c_N$ as possible actions, a termination condition $\beta = 1$ if $Tick(c_i) \in \{SUCCESS, FAILURE\}$, and a policy $\mu$ to be learned.
\end{definition}

and

\begin{definition}
A Learning Action Node can be seen as an option with $\mathcal{I} = \mathcal{S}$, with actions $a \in A_s$, a termination condition $\beta$, and a policy $\mu$ to be learned.
\end{definition}

As result from Theorem \ref{theorem:options}, we can also exploit the features provided from the Options framework \cite{sutton1998}, such as: both learning nodes can be trained using Q-Learning; guarantee of convergence for nested nodes with the same conditions to a single Q-Learning model; due to the division of the space, the nodes can converge faster than a single learning model; nodes can be interrupted by prioritized behaviors without problem; intra-option learning can be used to speed up global convergence among the tree.

\section{Experimental Validation}
\label{sec:experiments}

In this Section, we present two simulated fire control scenarios to validate empirically the proposed framework. The following experiments use discrete state and action Q-Learning in the Learning Nodes. However, continuous versions of Q-Learning could be used. We also used the Behavior3 library and editor \cite{pereira2014} for modeling the Behavior Trees. All experiments and custom code are available online\footnote{All code will be available in http://renatopp.com/research after revision.}.

Both experiments execute 30 trials, each with 400 iterations. At the beginning of each trial the experiment is reset. All charts show the average results of the 30 trials. The environment is divided into infinite rooms. In each room, the agent can perform 3 possible actions: \emph{save victim}, \emph{use extinguisher X} and \emph{change room}. Any given room has 50\% of chance to have a victim in the wreckage; if there is a victim, the agent must save it first. A room also has 50\% of chance to have one of 3 types of fire (types $1, 2$ and $3$, with $1/3$ chance each); if there is a fire, the agent must extinguish it before leaving the room and after saving the victim. There are 3 types of extinguishers (types $A, B$ and $C$), each extinguisher can extinguish only one type of fire, randomly chosen in the beginning of the trial; this map is unknown to the agent. If the room has no victim and no fire, the agent must go to the next room. If, at any moment, the wrong extinguisher is used, the room is lost and the agent must change to the next room.

\subsection{Scenario 1}
\label{sec:experiments_scenario1}

\begin{figure}
  \centering
  \includegraphics[width=0.85\textwidth]{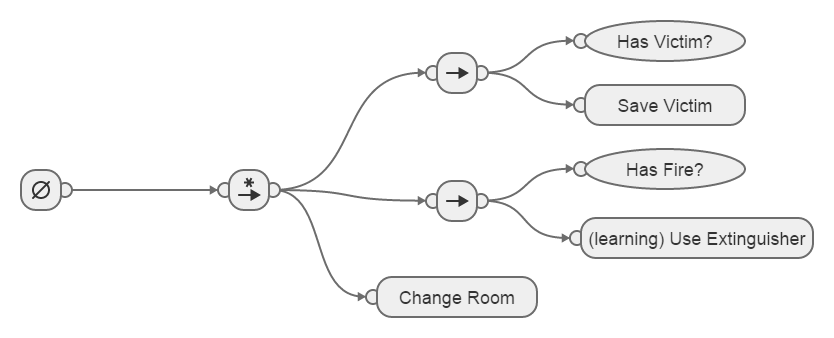}
  \caption{The Behavior Tree model for the first agent. Only a single action node (Use Extinguisher) has learning capabilities.}
  \label{fig:model_scenario1}
\end{figure}

In this scenario, all actions are instantly. Figure \ref{fig:model_scenario1} shows the Behavior Tree that models the learning agent. Notice that, after the root child, there are three branches: the first, with highest priority, represents the behavior \emph{save victim}; the second branch, represents \emph{use extinguisher X}; and the last branch, with lowest priority, represents \emph{change room}.

We use a learning action node for \emph{use extinguisher X} behavior. This node must learn which extinguisher can be used for each fire type; it is configured to receive the state $s = \langle \text{\emph{fire type}} \rangle$, where $\text{\emph{fire type}} = \{1, 2, 3\}$, and with actions $a = \{A, B, C\}$; this node also receives the reward of $+10$ if it could extinguish the fire and $-10$ otherwise.

Table \ref{tab:results_interference} shows the ratio between the correct activations of the three main behaviors over the total expected activations (i.e., the accuracy for behavior usage). All behaviors are called correctly 100\% of the time, this is due to the tree dynamics that allow the expert to model a strict sequence of behaviors. Notice that, this table also shows that the learning node (Use Extinguisher) does not affect the execution of other behaviors in the tree. Figure \ref{fig:results_interference} shows the convergence of the node's accuracy during the experiment, compared with the random baseline.

\begin{figure}
  \centering
  \includegraphics[width=0.85\textwidth]{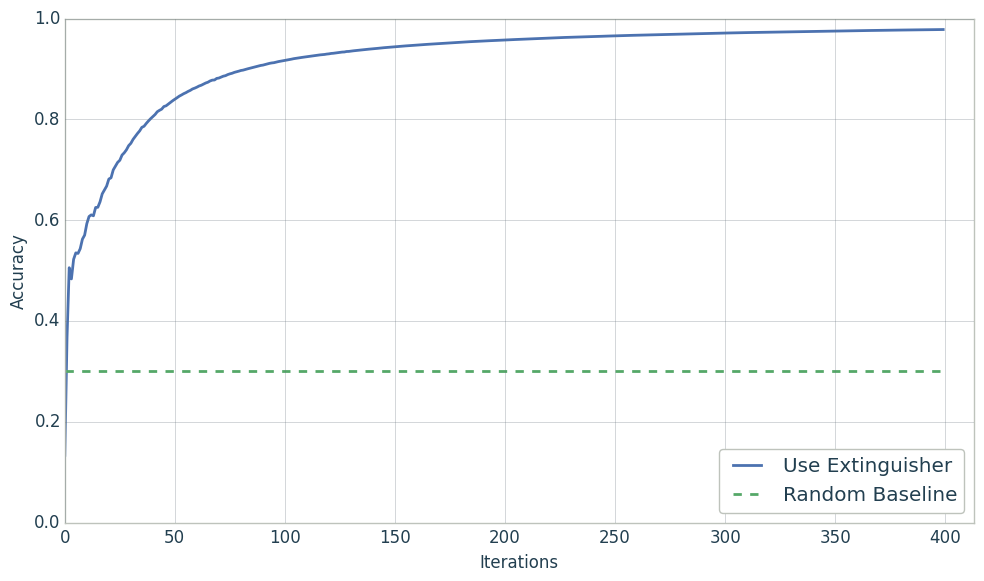}
  \caption{Accuracy of the ``Use Extinguisher'' node (a learning action node) compared with the random baseline.}
  \label{fig:results_interference}
\end{figure}

\begin{table}
  \centering
  \caption{Accuracy of behavior usage.}
  \label{tab:results_interference}
  \begin{tabular}{ l c }
    \toprule
    \textbf{Behavior} & \textbf{Accuracy} \\
    \hline 
    Save Victim       & 1.0 \\
    Use Extinguisher  & 1.0 \\
    Change Room       & 1.0 \\
    \bottomrule
  \end{tabular}
\end{table}

\subsection{Scenario 2}
\label{sec:experiments_scenario2}

In this scenario, we add more complexity to the task. The agent actions \emph{save victim} and \emph{use extinguisher X} now take time to complete, depending on the fire intensity. Any given fire has an intensity $\text{\emph{fire intensity}} \in \{1, 2, 3\}$, chosen randomly for each room. The fire intensity specifies how many ticks the agent needs to perform the actions (i.e., when fire intensity is $0$, all actions are instantly; when it is $1$, actions take 1 tick to be completed; and so on). The fire intensity is reduced by $1$ each tick when the right extinguisher is being used. Notice that \emph{change room} is always instantly and the use of the wrong extinguisher makes the agent lose the room. 

\begin{figure}
  \centering
  \includegraphics[width=0.85\textwidth]{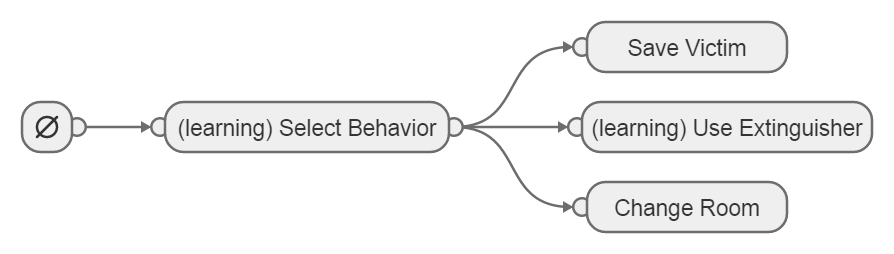}
  \caption{The Behavior Tree model for the second agent using two nested Learning Nodes.}
  \label{fig:model_scenario2}
\end{figure}

Figure \ref{fig:model_scenario2} shows the Behavior Tree that models the learning agent for this scenario. Now, it uses 2 learning nodes. The first, similar to the one used in the first scenario, is a learning action using the state $s = \langle \text{\emph{fire type}} \rangle$, where $\text{\emph{fire type}} = \{1, 2, 3\}$, and actions $a = \{A, B, C\}$; this node receives the reward of $\frac{10}{\text{\emph{fire intensity}}}$ while using the right extinguisher (i.e., $+10$ when the fire is extinguished), and $-10$ if the wrong extinguisher is used.

\begin{table}
  \centering
  \caption{Accuracy of behavior usage.}
  \label{tab:results_nested}
  \begin{tabular}{ l c }
    \toprule
    \textbf{Behavior} & \textbf{Accuracy} \\
    \hline 
    Save Victim       & 0.974 \\
    Use Extinguisher  & 0.991 \\
    Change Room       & 0.991 \\
    \bottomrule
  \end{tabular}
\end{table}

The second learning node is the root child, which must learn which behavior must be executed given the state $s = \langle \text{\emph{has victim?}}, \text{\emph{has fire?}}\rangle$. The node's children are the actions $a = \{\text{\emph{save victim}}, \text{\emph{use extinguisher}}, \text{\emph{change room}}\}$. This node receives the rewards: $-10$ if the node tries to save and there is no victim, $-1$ while saving the victim, and $+10$ when the victim is saved; $-10$ if trying to extinguish a non-existing fire, $-1$ while extinguishing it, and $+10$ when the fire is extinguished; and $+10$ when the agent leaves the room at the right moment and $-10$ otherwise.

\begin{figure}
  \centering
  \includegraphics[width=0.85\textwidth]{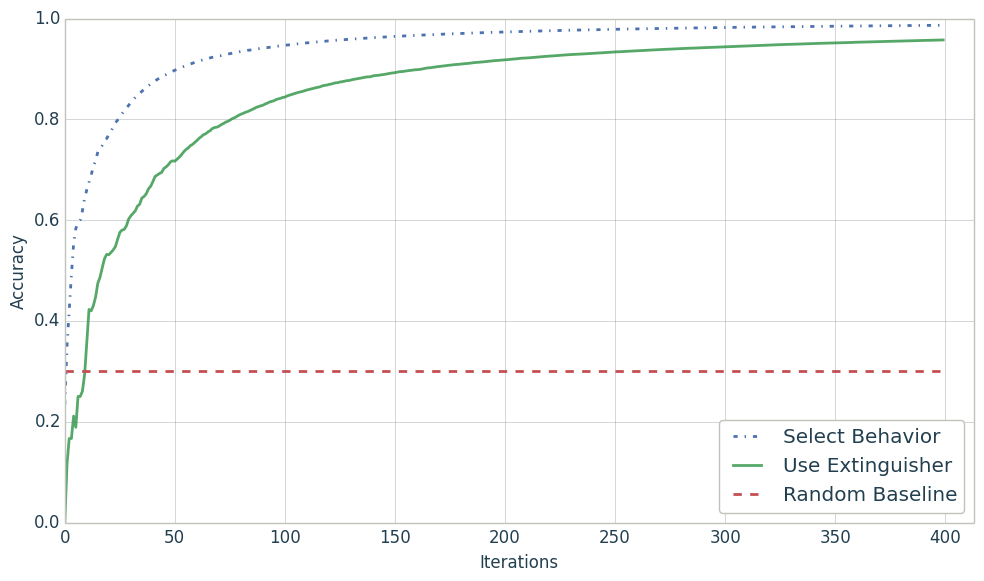}
  \caption{Accuracy of the composite and action Learning Nodes.}
  \label{fig:results_nested}
\end{figure}

Again, Table \ref{tab:results_nested} shows the accuracy of the behavior usage. In this case, behaviors are called correctly 97\% to 99\% of the time, differing from the previous result in scenario 1. This difference is because the learning process needs a step of trial and error to learn the most effective behavior. Figure \ref{fig:results_nested} shows the convergence of both node's accuracy during the experiment compared with the random baseline.

\section{Related Work}
\label{sec:related}

Behavior Trees were created as alternative to Hierarchical Finite State Machines (HFSMs) and similar methods, aiming to provide more flexible controller for Non-Playable Characters (NPCs) in video games \cite{isla2005}. The method had a quick acceptance in the game industry and, recently, it has been applied to robotics \cite{ogren2012} \cite{marzinotto2014} \cite{colledanchise2014a} \cite{colledanchise2014b} \cite{colledanchise2015}, where BTs received a more formal and standard definition.

We base our developments on the work of \nocite{marzinotto2014}(Marzinotto et al., 2014), where the authors prove the relation of Behaviors Trees to Controlled Hybrid Dynamical Systems (CHDSs). Our work has some differences to theirs. Firstly, we define the Behavior Tree as a Directed Rooted Tree (DRT) and not as a Directed Acyclic Graph (DAG) because we consider that multi parenting is only done in the implementation level, not in the modeling level. We also prefer the term \emph{Priority} over \emph{Selector}, due to the similarity to the \emph{Sequence} node. We also consider \emph{MemSequence} and \emph{MemPriority} as core nodes and not extension to core nodes. Different from almost all descriptions of BTs, we represent the model graphically in a horizontal layout due to legibility and space. In our layout, the priority is given from top to bottom.

Naturally, this is not the first time a custom node is proposed to Behavior Trees. For example, \nocite{johansson2012}(Johansson et al., 2012) proposed an emotional node which uses an ``emotion function'' to sort the node's children according to the agents feelings; \nocite{palma2011}(Palma et al., 2011) and \nocite{florez2009}(Florez-Puga et al., 2009) propose a query node that looks for possible subtrees in a Case-Based Reasoning (CBR) systems, resulting in dynamic trees. As far as we known, there is no use of a custom learning node neither the use of Reinforcement Learning with Behaviors Trees in the current literature.

We show that our framework has a close relation to the Options framework \cite{sutton1999}, but it has also similarities to other models of Hierarchical Reinforcement Learning, such as the Hierarchies of Abstract Machines \cite{parr1998} and the MAXQ model \cite{dietterich2000}. As a general case, authors in Hierarchical Reinforcement Learning area see the manual division of behaviors as a problem to be dealt while we use this as an intrinsic part of our approach, i.e., the manual definition of behaviors is viewed as a mean to use prior and expert knowledge of the problem.

\section{Conclusion}
\label{sec:conclusion}

We have proposed a framework to use Reinforcement Learning in behavior-based agents, providing adaptiveness to physical or virtual agents while respecting the constraints modeled by the expert. Based on Behavior Trees, we proposed the creation of a new type of Composite and Action node, called \emph{Learning Node}, in which we embed a Q-Learning algorithm to perform a local learning, without affecting how other nodes work. We show that this framework is related to Hierarchical Reinforcement Learning, being a specialization of the Options framework, thus ensuring convergence of nested learning nodes, allowing them to be interrupted before the task is completed and allowing the use of intra-option learning for more complex models. 

We also validate our framework empirically using experiments in simulated fire control scenarios. The experiments show how to use the expert knowledge to model the behavior choices without interference of the learning nodes, and confirms that nested learning nodes can converge and work with temporal actions.

This framework provides the formalization needed to expand the research on adaptive and constrained behavior-based agents using Behavior Trees and Reinforcement Learning. We expect to further improve this framework by extending it with capabilities for working in non-stationary and continuous state space, allowing us to create agents for more complex environments.

\section*{Acknowledgment}
The authors thank Edigleison Carvalho and Thiago Rodrigues for their valuable input on this paper. This work is supported by CNPq, a Brazilian government entity for scientific and technological development.

\bibliographystyle{splncs03}
\bibliography{arxiv2015}

\end{document}